\documentclass{llncs}

\usepackage[usenames,dvipsnames,svgnames,table]{xcolor}
\usepackage{latexsym}
\usepackage[ruled,vlined]{algorithm2e}
\usepackage{tikz}
\usepackage{xspace}
\usepackage{amssymb}
\usepackage{subfig}
\usepackage{multirow}
\usepackage{rotating}
\usepackage{accents}
\usepackage{pifont}
\usepackage{varwidth}
\usepackage{enumerate}

\usepackage{tikz}
\usepackage{fp}

\usetikzlibrary{arrows,shadows,fit,calc,positioning,decorations.pathreplacing,matrix,shapes,petri,topaths,fadings,mindmap,backgrounds}

\newcommand{\texmath}[1]{\ensuremath{#1}\xspace}

\newcommand{\constraint}[1]{\mbox{\sc #1}\xspace}
\newcommand{\algo}[1]{\mbox{\texttt{#1}}}

\SetKwInput{KwComplexity}{Complexity}

\def\vc{\constraint{VertexCover}}
\def\nvalue{\constraint{NValue}}

\newcommand\thickbar[1]{\accentset{\rule{.5em}{.8pt}}{#1}}
\newcommand{\ub}[1]{\texmath{\thickbar{#1}}}
\DeclareRobustCommand{\lb}[1]{\texmath{\underaccent{\thickbar{}}{#1}}}

\def\instance{P}

\def\np{NP}

\def\witness{\omega}
\def\limit{\lambda}

\def\asetvar{S}

\def\agraph{G}

\def\vertices{V}
\def\somevertices{W}
\def\edges{E}
\def\acover{\asetvar}
\def\astable{I}
\def\anotherstable{J}
\def\aclique{C}
\def\amatching{M}
\def\anedge{e}
\newcommand{\edgeb}[2]{\texmath{\{{#1},{#2}\}}}

\def\avertex{v}
\def\overtex{u}

\def\avalue{v}

\def\aset{s}
\def\acrown{W}
\def\instance{x}

\def\costvar{K}
\def\param{k}
\def\z{z}
\def\universe{U}
\def\restricted{R}
\def\forced{F}
\def\indifferent{I}
\def\residue{H}

\def\vcover{S}
\def\ccover{T}

\def\avar{X}
\def\ovar{Y}
\def\acons{R}
\def\atuple{\tau}
\newcommand{\valueof}[2]{\texmath{{#1}[{#2}]}}

\newcommand{\truekernelized}[1]{\texmath{{#1}^{k}}}
\newcommand{\rigikernelized}[1]{\texmath{{#1}^{r}}}

\def\truekernel{\texmath{\truekernelized{\residue},\truekernelized{\acrown}}} 
\def\rigikernel{\texmath{\rigikernelized{\residue},\rigikernelized{\forced}}} 

\newcommand{\subgraph}[2]{{#1[{#2}]}}

\newcommand{\neighbors}[1]{\texmath{N({#1})}}
\newcommand{\neighborsp}[1]{\texmath{N^+({#1})}}

\newcommand{\dom}[1]{{\cal D}({#1})}

\newcommand{\post}[3]{{#1}[{#2}]({#3})}

\def\HiLi{\leavevmode\rlap{\hbox to \hsize{\color{yellow!30}\leaders\hrule height .8\baselineskip depth .5ex\hfill}}}

\def\HiLiO{\leavevmode\rlap{\hbox to \hsize{\color{orange!30}\leaders\hrule height .8\baselineskip depth .5ex\hfill}}}

\newcommand{\cut}[1]{}

\title{Propagation via Kernelization: the Vertex Cover Constraint}

\author{
Cl\'ement Carbonnel \and
Emmanuel Hebrard 
}

\institute{
LAAS-CNRS, Universit\'e de Toulouse, CNRS, INP, Toulouse, France \\
}

\begin{document}

\maketitle

\begin{abstract}
	The technique of kernelization consists in extracting, from an instance of a problem, an essentially equivalent instance whose size is bounded in a parameter $\param$. Besides being the basis for efficient parameterized algorithms, this method also provides a wealth of information to reason about in the context of constraint programming. We study the use of kernelization for designing propagators through the example of the Vertex Cover constraint. Since the classic kernelization rules often correspond to dominance rather than consistency, we introduce the notion of ``loss-less'' kernel. While our preliminary experimental results show the potential of the approach, they also show some of its limits. In particular, this method is more effective for vertex covers of large and sparse graphs, as they tend to have, relatively, smaller kernels. 
\end{abstract}

\section{Introduction}
\label{sec:intro}

The fact that there is virtually no restriction on the algorithms used to reason about each constraint was critical to the success of constraint programming. For instance, efficient algorithms from matching and flow theory~\cite{networkflow:book,hopcroft73:alg} were adapted as \emph{propagation} algorithms~\cite{regin94:fil,hoeve06:glo} and subsequently lead to a number of successful applications. 
\np-hard constraints, however, are often simply decomposed. 
Doing so may significantly hinder the reasoning made possible by the knowledge on the structure of the problem.
For instance, finding a support for the \nvalue\ constraint is \np-hard, yet enforcing some incomplete propagation rules for this constraint has been shown to be an effective approach~\cite{bessiere06:fil,fages13:fil}, compared to decomposing it,
or enforcing bound consistency~\cite{DBLP:conf/cp/Beldiceanu01}.

The concept of parameterized complexity is very promising in the context of propagating \np-hard constraints. 
A study of the parameterized complexity of global constraints~\cite{bessiere08:par}, and of their pertinent parameters, showed that they were a fertile ground for this technique.
For instance, a kernelization of
the \nvalue\ constraint was introduced in~\cite{DBLP:conf/ijcai/GaspersS11}, yielding an FPT consistency algorithm.
A kernel is an equivalent instance of a problem whose size is bounded in a parameter $\param$. If a problem has a polynomial-time computable kernel, then it is FPT since brute-force search on the kernel can be done in time $O^*(f(\param))$ for some computable function $f$. Moreover, kernelization techniques can provide useful information about suboptimal and/or compulsory choices, which can be used to propagate. In this paper we consider the example of the \emph{vertex cover} problem, where we want to find a set of at most $\param$ vertices $\acover$ of a graph $\agraph=(\vertices,\edges)$ such that every edge of $\agraph$ is incident to at least one vertex in $\acover$. This problem is a long-time favourite of the parameterized complexity community and a number of different kernelization rules have been proposed, along with very efficient FPT algorithms (the most recent being the $O(1.2738^k + k|\vertices|)$ algorithm by Chen, Kanj and Xia~\cite{conf/mfcs/ChenKX06}).

Since the complement of a minimum vertex cover is a maximum independent set, a \vc constraint can also be used to model variants of the maximum independent set and maximum clique problems with side constraints modulo straightforward modeling tweaks. Among these three equivalent problems, vertex cover offers the greatest variety of pruning techniques and is therefore the most natural choice for the definition of a global constraint.
%
%
Through this example, we highlight the ``triple'' value of kernelization in the context of constraint programming: 

First, some kernelization rules are, or can be generalized to, filtering rules. Since the strongest kernelization techniques rely on dominance they cannot be used directly for filtering. Therefore, we introduce the notion of \emph{loss-less} kernelization which preserves all solutions and can thus be used in the context of constraint propagation. Moreover, we show that we can use a more powerful form of kernel, the so-called \emph{rigid crowns} to effectively filter the constraint when the lower bound on the size of the vertex cover is tight. We discuss the various kernelization techniques for this problem in Section~\ref{sec:kern}.

Second, even when it cannot be used to filter the domain, a kernel can be sufficiently small to speed up lower bound computation, or to find a ``witness solution'' and sometimes an exact lower bound. We also show that such a support can be used to obtain stronger filtering.
We introduce a propagation algorithm based on these observations in Section~\ref{sec:algo}.
Along this line, the kernel could also be used to guide search, either using the witness solution or the dominance relations on variable assignments.

Third, because a kernel garantees a size at most $f(\param)$ for a parameter $\param$, one can efficiently estimate the likelihood that these rules will indeed reduce the instance. We report experimental results on a variant of the vertex cover problem in Section~\ref{sec:expe}. These experiments show that, as expected, kernelization techniques perform better when the parameter is small. However, we observe that the overhead is manageable, even in unfavorable cases.
Moreover, one could dynamically choose whether costly methods should be applied by comparing the value of the parameter $\param$ (in our case, the upper bound of the variable standing for the size of the cover) to the input size.

%
%
%
%
%
%

\section{Background and Notations}
\label{sec:back}

An \emph{undirected graph} is an ordered pair $\agraph = (\vertices,\edges)$ where $\vertices$ is a set of vertices and $\edges$ is a set of edges, that is, pairs in $\vertices$. We denote the \emph{neighborhood} $\neighbors{\avertex} = \{\overtex \mid \edgeb{\avertex}{\overtex} \in \edges\}$ of a vertex $\avertex$, its \emph{closed neighborhood} $\neighborsp{\avertex} = \neighbors{\avertex} \cup \{\avertex\}$ and $\neighbors{\somevertices} = \bigcup_{\avertex \in \somevertices}\neighbors{\avertex}$. The \emph{subgraph} of $\agraph = (\vertices,\edges)$ induced by a subset of vertices $\somevertices$ is denoted $\subgraph{\agraph}{\somevertices} = (\somevertices, 
2^\somevertices \cap \edges)$. 
%
An \emph{independent set} is a set $\astable \subseteq \vertices$ such that no pair of elements in $\astable$ is in $\edges$. 
A \emph{clique} is a set $\aclique \subseteq \vertices$ such that every pair of elements in $\aclique$ is in $\edges$. 
A \emph{clique cover} $\ccover$ of a graph $\agraph  = (\vertices,\edges)$ is a collection of disjoint cliques such that $\bigcup_{\aclique \in \ccover}\aclique = \vertices$.
A \emph{matching} is a subset of pairwise disjoint edges.
A \emph{vertex cover} of $\agraph$ is a set $\acover \subseteq \vertices$ such that every edge $\anedge \in \edges$ is incident to at least one vertex in $\vcover$, i.e., $\acover \cap \anedge \neq \emptyset$. The \emph{minimum vertex cover problem} consists in finding a vertex cover of minimum size. Its decision version \cut{below} is \np-complete~\cite{Garey:1979:CIG:578533}.

The standard algorithm for solving this problem is a simple branch and bound procedure.
There are several bounds that one can use, in this paper we consider the minimum clique cover of the graph (or, equivalently, a coloring of its complement). 
Given a clique cover $\ccover$ of a graph $\agraph=(\vertices,\edges)$, we know that all but one vertices in each clique of $\ccover$ must be in any vertex cover of $\agraph$. Therefore, $|\vertices|-|\ccover|$ is a lower bound of the size of the minimum vertex cover of $\agraph$. 
%
The algorithm branches by adding a vertex to the cover (left branch) or adding its neighborhood to the cover (right branch).

\medskip

A \emph{constraint} is a predicate over one or several variables.
In this paper we consider the vertex cover problem as a constraint over two variables:
an \emph{integer} variable $\costvar$ to represent the bound on the size of the vertex cover,
and a \emph{set variable} $\asetvar$ to represent the cover itself.
The former takes integer values in a \emph{domain} $\dom{\costvar}$ which minimum and maximum values are denoted $\lb{\costvar}$ and $\ub{\costvar}$, respectively.
The latter takes its values in the sets that are supersets of a \emph{lower bound} $\lb{\asetvar}$ and subsets of an \emph{upper bound} $\ub{\asetvar}$. Moreover, the domain of a set variable is also often constrained by its cardinality given by an integer variable $|\asetvar|$. 
\cut{The domain of a set variable is therefore the set of sets $\dom{\asetvar} = \{\aset \mid \lb{\asetvar} \subseteq \aset \subseteq \ub{\asetvar} ~\&~ \lb{|\asetvar|} \leq |\aset| \leq \ub{|\asetvar|}\}$.}
%
%
We consider a constraint on these two variables and whose predicate is the vertex cover problem on the graph $\agraph = (\vertices,\edges)$ given as a parameter:
\begin{definition}[\vc\ constraint]
	
  $\post{\vc}{\agraph}{\costvar,\asetvar} \iff |\asetvar| \leq \costvar ~\&~ \forall \edgeb{\avertex}{\overtex} \in \edges,~ \avertex \in \asetvar \vee \overtex \in \asetvar$
\end{definition}

\cut{
We say that a set $\aset$ is \emph{bound valid} for a set variable $\asetvar$ iff $\aset \in \dom{\asetvar}$, and an integer $\avalue$ is bound valid for an integer variable $\avar$ iff $\lb{\avar} \leq \avalue \leq \ub{\avar}$.
Given a constraint $\acons$ constraining a variable $\asetvar$, we say that the \emph{assignment} $\avertex \in \asetvar$ (resp. $\avertex \not\in \asetvar$) is \emph{bound consistent} iff there exists a \emph{bound support}, that is, a tuple $\atuple \in \acons$ such that $\avertex \in \valueof{\atuple}{\asetvar}$ (resp. $\avertex \not\in \valueof{\atuple}{\asetvar}$),
and for each $\avar$ constrained by $\acons$, $\valueof{\atuple}{\avar}$ is bound valid for $\avar$. 
Similarly, for an integer variable $\ovar$ constrained by $\acons$, we say that the assignment $\ovar = \avalue$ is bound consistent iff there exists $\atuple \in \acons$ with $\valueof{\atuple}{\ovar} = \avalue$ and for every $\avar$, $\valueof{\atuple}{\avar}$ is bound valid for $\avar$.

A constraint $\acons$ is bound consistent iff, for every set variable $\asetvar$, the assignment $\avertex \in \asetvar$ is bounds consistent iff $\avertex \in \ub{\asetvar}$ and $\avertex \not\in \asetvar$ is bound consistent 
iff $\avertex \not\in \lb{\asetvar}$ and for every integer variable $\ovar$, the assignments $\ovar = \lb{\ovar}$ and $\ovar = \ub{\ovar}$ are bounds consistent.
}

A \emph{bound support} for this constraint is a solution of the \vc problem. Since enforcing \emph{bound consistency} would entail proving the existence of two bound supports for each element in $\ub{\asetvar} \setminus \lb{\asetvar}$ and one for the lower bound of $\costvar$, there is no polynomial algorithm unless P=NP. 
In this paper we consider pruning rules that are not complete with respect to usual notions of consistencies.



%
%
%
%
%
%



\section{Kernelization as a Propagation Technique}
\label{sec:kern}

\subsection{Standard Kernelization}
\label{ssec:skern}

A problem is \textit{parameterized} if each instance $\instance$ is paired with a nonnegative integer $\param$, and a parameterized problem is \textit{fixed-parameter tractable} (FPT) if it can be solved in time $O(|\instance|^{O(1)}f(\param))$ for some function $f$. A \textit{kernelization algorithm} takes as input a parameterized instance $(\instance,\param)$ and creates in polynomial time a parameterized instance $(\instance',\param')$ of the same problem, called the \textit{kernel}, such that
\begin{center}
\begin{varwidth}{\textwidth}
\begin{enumerate}[(i)]
\item $(\instance',\param')$ is satisfiable if and only if $(\instance,\param)$ is satisfiable;
\item  $|\instance'| \leq g(\param)$ for some computable function $g$, and 
\item $\param' \leq h(\param)$ for some computable function $h$.
\end{enumerate}
\end{varwidth}
\end{center}

While this formal definition does not guarantee that the kernel is a subinstance of $(\instance,\param)$, in graph theory kernelization algorithms often operate by applying a succession of dominance rules to eliminate vertices or edges from the graph.
In the case of vertex cover, the simplest dominance rule is the \textit{Buss rule}: if a vertex $v$ has at least $\param+1$ neighbors, then $v$ belongs to every vertex cover of size at most $\param$; we can therefore remove $v$ from the graph and reduce $\param$ by one. Applying this rule until a fixed point yields an elementary kernel that contains at most $\param^2$ edges and $2\param^2$ non isolated vertices~\cite{buss1993nondeterminism}. A more refined kernelization algorithm works using structures called crowns. A \textit{crown} of a graph $\agraph=(\vertices,\edges)$ is a partition $(\residue,\acrown,\astable)$ of $\vertices$ such that 

\begin{center}
\begin{varwidth}{\textwidth}
\begin{enumerate}[(i)]
\item $\astable$ is an independent set;
\item There is no edge between $\astable$ and $\residue$, and
\item There is a matching $\amatching$ between $\acrown$ and $\astable$ of size $|\acrown|$.
\end{enumerate}
\end{varwidth}
\end{center}

Every vertex cover of $\subgraph{\agraph}{\acrown \cup \astable}$ has to be of size at least $|\acrown|$ because of the matching $\amatching$. Since $\astable$ is an independent set, taking the vertices of $\acrown$ over those of $\astable$ into the vertex cover is always a sound choice: they would cover all the edges between $\acrown$ and $\astable$ at minimum cost and as many edges in $\subgraph{\agraph}{\residue \cup \acrown}$ as possible. A simple polynomial-time algorithm that finds a crown greedily from a maximal matching already leaves an instance $\subgraph{\agraph}{\residue}$ with at most $3\param$ vertices~\cite{abu2007crown}.
A stronger method using linear programming yields a (presumably optimal) kernel of size $2\param$~\cite{nemhauser1975vertex}.

\subsection{Loss-less Kernelization}
\label{ssec:extbuss}

The strongest kernelization rules correspond to dominance relations rather than inconsistencies. However, the Buss rule actually detects inconsistencies, that is, vertices that must be in the cover.
We call this type of rules \emph{loss-less} as they do not remove solutions. We can extend this line of reasoning by considering rules that do not remove solutions close to the optimum: for the \vc constraint, the variable $\costvar$ is likely to be minimized and the situation where all solutions are close to the optimum will inevitably arise. This can be formalized in the context of \textit{subset minimization problems}, which ask for a subset $\asetvar$ with some property $\pi$ of a given universe $\universe$ such that $|\asetvar| \leq k$. In the next definition we denote by \texttt{opt} the cardinality of a minimum-size solution.

\begin{definition}
Given an integer $\z$ and a subset minimization problem parameterized by solution size $\param$, a $\z$-loss-less kernel is a partition $(\residue,\forced,\restricted,\indifferent)$ of the universe $U$ where 
\begin{itemize}
\item $\forced$ is a set of \emph{forced} items, included in \textit{every} solution of size at most \texttt{opt}+$z$; 
\item $\restricted$ is a set of \emph{restricted} items, intersecting with no
solution of size at most \texttt{opt}+$z$; 
\item $\residue$ is a residual problem, whose size is bounded by a function in $\param$ and 
\item $\indifferent$ is a set of \emph{indifferent} elements, i.e., if $i \in \indifferent$, then $\phi$ is a solution of size at most $\param-1$ if and only if $\phi \cup i$ is a solution. 
\end{itemize}
\end{definition}


An $\infty$-loss-less kernel is simply said to be \textit{loss-less}. The Buss kernel is a loss-less kernel for vertex cover that never puts any vertices in $\restricted$ ($\forced$ contains vertices of degree strictly greater than $\param$, and $\indifferent$ contains isolated vertices). In the case of vertex cover, the set $R$ is always empty unless $z=0$. Note that \textit{loss-free} kernels introduced in the context of backdoors~\cite{samer2008backdoor} are different since they only preserve \textit{minimal} solutions; for subset minimization problems those kernels are called \textit{full kernels}~\cite{Damaschke2006337}.

A kernel for vertex cover that preserves all minimum-size solutions has been introduced in~\cite{Chlebík2008292}. In our terminology, this corresponds to a $0$-loss-less kernel. Interestingly, this kernelization is based on a special type of crown reduction but yields a kernel of size $2\param$ (matching the best known bound for standard kernelization). The idea is to consider only crowns $(\residue,\acrown,\astable)$ such that $\acrown$ is the \textit{only} minimum-size vertex cover of $\subgraph{\agraph}{\acrown \cup \astable}$, as for this kind of crown vertices of $\acrown$ are always a \textit{strictly} better choice that those of $\astable$. Those crowns are said \textit{rigid}. The authors present a polynomial-time algorithm that finds the (unique) rigid crown $(\residue,\acrown,\astable)$ such that $\residue$ is rigid crown free and has size at most $2\param$. 
Their algorithm 
works as follows. First,
build from $\agraph=(\vertices,\edges)$ the graph $B_G$ with two vertices $v_l,v_r$ for every $v \in V$ and two edges $\edgeb{v_l}{u_r},\edgeb{u_l}{v_r}$ for every edge $\edgeb{v}{u} \in \edges$. Compute a maximum matching $\amatching$ of $B_G$ (which can be done in polynomial time via the Hopcroft-Karp algorithm~\cite{hopcroft73:alg}). Then, if $D$ is the set of all vertices in $B_G$ that are reachable from unmatched vertices via $\amatching$-alternating paths of even length, a vertex $v$ in $\agraph$ belongs to the independent set $\astable$ of the rigid crown if and only if $v_l$ and $v_r$ belong to $D$.
This algorithm is well suited to constraint propagation as bipartite matching algorithms based on augmenting paths are efficient and incremental.



\subsection{Witness Pruning}
\label{ssec:witness}

Last, even if the standard kernel uses dominance relations, it can indirectly be used for pruning. By reducing the size of the problem it often makes it possible to find an optimal vertex cover
relatively efficiently. This vertex cover gives a valid (and maximal) lower bound.
Moreover, given an optimal cover $\acover$ 
we can find inconsistent values by asserting that some vertices must be in any cover of a given size.

\begin{theorem}
	\label{thm:wpruning}
	if $\acover$ is an optimal vertex cover of $\agraph = (\vertices, \edges)$ such that there exists $\avertex \in \acover, \anotherstable \subseteq \neighbors{\avertex} \setminus \acover$ with $\neighbors{\anotherstable} \subseteq \neighborsp{\avertex}$ then any vertex cover of $\agraph$ either contains $\avertex$ or at least $|\acover|+|\anotherstable|-1$ vertices.
\end{theorem}

\begin{proof}
Let $\param$ be an upper bound on the size of the vertex cover, $\avertex \in \acover$ be a vertex in an optimal vertex cover $\acover$. Consider $\anotherstable \subseteq \neighbors{\avertex} \setminus \acover$ such that 
$\neighbors{\anotherstable} \subseteq \neighborsp{\avertex}$. Suppose there exists a vertex cover $\acover'$ such that $|\acover'| < |\acover|+|\anotherstable|-1$ and $\avertex \notin \acover'$. $\acover'$ must contain every node in $\neighbors{\avertex}$ and hence in $\anotherstable$. However, we can build a vertex cover of size at most $|\acover|-1$ by replacing $\anotherstable$ by $\avertex$, since $\vertices \setminus \acover$ and thus $\anotherstable$ are independent sets.
\qed
\end{proof}

If we can manage to find a minimum vertex cover $\acover$, for instance when the kernel is small enough so that it can be explored exhaustively, Theorem~\ref{thm:wpruning} entails a pruning rule. 
If we find a vertex $\avertex \in \acover$ and a set $\anotherstable \in \neighbors{\avertex} \setminus \acover$ with $\neighbors{\anotherstable} \subseteq \neighborsp{\avertex}$ and $|\anotherstable| > \param - |\acover|$ then we know that $\avertex$ must be in all vertex covers of size $\leq \param$.

\section{A Propagation Algorithm for \vc}
\label{sec:algo}

In this section we give the skeleton of a propagation algorithm for the \vc constraint based on the techniques discussed in Section~\ref{sec:kern}.
\begin{algorithm}[htbp]
	\begin{small}
	\caption{\label{algo::vcpropag}\algo{PropagateVertexCover}($\asetvar, \costvar, \agraph=(\vertices,\edges), \limit, \witness$)}
	
	\lnl{ln:basepruning}$\lb{\asetvar} \gets \lb{\asetvar} \cup \neighbors{\vertices \setminus \ub{\asetvar}}$\;


	\lnl{ln:nkernel}$\rigikernel \gets$ \algo{BussKernel}($\subgraph{\agraph}{\ub{\asetvar} \setminus \lb{\asetvar}}$)\;

	\lnl{ln:witnesscheck}\If{$\witness \not\subseteq \ub{\asetvar} ~\vee~ |\witness \cup \lb{\asetvar}| \geq \ub{\costvar}$} {
	
	\HiLiO\lnl{ln:rkernel}$\truekernel \gets$ \algo{Kernel}($\rigikernelized{\residue}$)\;
		
	\lnl{ln:bruteforce}\lIf{$\limit>0$}{$\witness \gets \rigikernelized{\forced} \cup \truekernelized{\acrown} \cup$ \algo{VertexCover}($\truekernelized{\residue}, \limit$)}
		
	\lnl{ln:opt}\lIf{$\witness$ is optimal} {
		$\lb{\costvar} \gets |\witness|$
	}
	\lnl{ln:approx}\lElse{
		$\lb{\costvar} \gets \max(\lb{\costvar}, |\rigikernelized{\forced}|+|\truekernelized{\forced}|+$\algo{LowerBound}$(\truekernelized{\residue}))$
	}		
	}

	\HiLiO\lnl{ln:tight}\If{$\lb{\costvar} = \ub{\costvar}$} {
	\HiLiO\lnl{ln:kernelpruning}$\rigikernel,\rigikernelized{\restricted} \gets$ \algo{RigidKernel}($\subgraph{\agraph}{\ub{\asetvar} \setminus \lb{\asetvar}}$)\;
	\HiLiO\lnl{ln:pruneub}$\ub{\asetvar} \gets \ub{\asetvar} \setminus \rigikernelized{\restricted}$\;
	}
	\HiLiO\lnl{ln:witnesspruning}\lElseIf{$\witness$ is optimal $\&~\ub{\costvar} - \lb{\costvar} \leq 2$} {
		$\lb{\asetvar},\ub{\asetvar} \gets$ \algo{WitnessPruning}($\agraph,  \witness$)
	}
	\lnl{ln:prunelb}$\lb{\asetvar} \gets \lb{\asetvar} \cup \rigikernelized{\forced}$\;
	\end{small}
\end{algorithm}

Algorithm~\ref{algo::vcpropag} takes as input the set variable $\asetvar$ standing for the vertex cover, an integer variable $\costvar$ standing for the cardinality of the vertex cover, and three parameters: the graph $\agraph=(\vertices,\edges)$, an integer $\limit$, and a ``witness'' vertex cover $\witness$ initialised to $\vertices$. 

The pruning in Line~\ref{ln:basepruning} is a straightforward application of the definition: the neighborhood of vertices not in the cover must be in the cover. 
Then, in Line~\ref{ln:nkernel}, we apply the $\infty$-loss-less kernelization (Buss rule) described in Section~\ref{ssec:extbuss} yielding a pair with a residual graph $\rigikernelized{\residue}$ and a set of nodes $\rigikernelized{\forced}$ that must be in the cover.

Next, if Condition~\ref{ln:witnesscheck} fails, there exists a vertex cover ($\witness \cup \lb{\asetvar}$) of size strictly less than $\ub{\costvar}$. As a result, the pruning from rigid crowns cannot apply. 
When the cover witness is not valid, we compute, in Line~\ref{ln:rkernel}, a standard kernel with the procedure \algo{Kernel}$(\agraph)$ using crowns, as explained in Section~\ref{ssec:skern}.
We then use this kernel to compute, in Line~\ref{ln:bruteforce},
a new witness using the procedure \algo{VertexCover}$(\agraph, \limit)$ which is the standard brute-force algorithm described in Section~\ref{sec:back}.
We stop the procedure when we find a vertex cover whose size is stricly smaller than the current upper bound, or when the search limit of $\limit$, in number of nodes explored by the branch \& bound procedure, is reached. In the first case, we know 
that the lower bound cannot be tight hence the constraint cannot fail nor prune further than the loss-less kernel. The second stopping condition is simply used to control the amount of time spent within the brute-force procedure.

If the call to the brute-force procedure was complete, we can conclude that the witness cover is optimal and therefore a valid lower bound (Line~\ref{ln:opt}). Otherwise, we simply use the lower bound computed at the root node by \algo{VertexCover}, denoted \algo{LowerBound} in Line~\ref{ln:approx}.
 If the lower bound is tight, then we can apply the pruning from rigid crowns as described in Section~\ref{ssec:extbuss}. Algorithm \algo{RigidKernel} returns a triple $\rigikernelized{\residue}, \rigikernelized{\forced}, \rigikernelized{\restricted}$ of residual, forced and restricted vertices, respectively.
Finally we apply a restriction to pairs of the pruning corresponding to Theorem~\ref{thm:wpruning} in Line~\ref{ln:witnesspruning}, and apply the pruning on the lower bound of $\asetvar$ corresponding to the forced nodes computed by \algo{BussKernel} and/or \algo{RigidKernel}.
%
%
%
%
%
%

%
%

\def\snap{\texttt{snap}\xspace}
\def\dimacs{\texttt{dimacs}\xspace}
\def\numpart{m}
\def\balance{b}

\def\clique{Cliques}
\def\decomp{Decomposition}
\def\onlylb{Clique Cover}
\def\kernelp{Kernel Pruning}
\def\witness{Kernel \& witness}
\def\all{\vc}
\def\allp{\vc$^+$}

\section{Experimental Evaluation}
\label{sec:expe}

We experimentally evaluated 
our propagation algorithm on the ``balanced vertex cover problem''. 
We want to find a minimum vertex cover which is balanced according to a partition of the vertices.
For instance, the vertex cover may represent a set of machines to shut down in a network so that all communications are interrupted. 
In this case, one might want to avoid shutting down too many machines of the same type, or same client, or in charge of the same service, etc.
By varying the
degree of balance
we can control the similarity of the problem to pure minimum vertex cover.
We used a range of graphs from the \dimacs and \snap repositories. For each graph $\agraph = (\vertices,\edges)$, we post a \vc constraint on the set variable
$\emptyset \subseteq \asetvar \subseteq \vertices$.
\begin{center}
\begin{table}[h!]
\caption{\label{tab:bvc} Comparison of approaches on the ``Balanced Vertex Cover'' problem.}
\tabcolsep=1.25pt
\begin{scriptsize}
\begin{tabular}{ll|rrrr|rrrr|rrrr|rrrr|rrrr}
\multicolumn{2}{c}{} & \multicolumn{4}{c}{\decomp} & \multicolumn{4}{c}{\onlylb} & \multicolumn{4}{c}{\kernelp} & \multicolumn{4}{c}{\witness} & \multicolumn{4}{c}{\all}\\
 &  & \#s & gap & cpu & \#nd & \#s & gap & cpu & \#nd & \#s & gap & cpu & \#nd & \#s & gap & cpu & \#nd & \#s & gap & cpu & \#nd\\
\hline
& \multicolumn{21}{c}{balancing constraint: tight}\\3 & \texttt{kel} & 2 & 2.00 & 9.7 & 0.4M & 2 & 2.00 & 10.6 & 0.2M & \cellcolor{TealBlue!30}{2} & \cellcolor{TealBlue!30}{2.00} & \cellcolor{TealBlue!30}{9.1} & \cellcolor{TealBlue!30}{0.2M} & 2 & 2.00 & 26.6 & 0.1M & 2 & 2.00 & 41.0 & 0.1M\\
15 & \texttt{p\_h} & 12 & 5.73 & 8.6 & 0.5M & \cellcolor{TealBlue!30}{10} & \cellcolor{TealBlue!30}{5.20} & \cellcolor{TealBlue!30}{15.6} & \cellcolor{TealBlue!30}{1.1M} & 11 & 5.20 & 11.2 & 0.6M & 11 & 4.67 & 27.7 & 0.4M & 11 & 4.67 & 28.8 & 0.4M\\
12 & \texttt{bro} & 9 & 3.67 & 0.1 & 11K & \cellcolor{TealBlue!30}{9} & \cellcolor{TealBlue!30}{3.67} & \cellcolor{TealBlue!30}{0.1} & \cellcolor{TealBlue!30}{4K} & \cellcolor{TealBlue!30}{9} & \cellcolor{TealBlue!30}{3.67} & \cellcolor{TealBlue!30}{0.1} & \cellcolor{TealBlue!30}{3K} & \cellcolor{TealBlue!30}{9} & \cellcolor{TealBlue!30}{3.67} & \cellcolor{TealBlue!30}{0.1} & \cellcolor{TealBlue!30}{2K} & \cellcolor{TealBlue!30}{9} & \cellcolor{TealBlue!30}{3.67} & \cellcolor{TealBlue!30}{0.2} & \cellcolor{TealBlue!30}{2K}\\
4 & \texttt{joh} & 1 & 0.00 & 0.1 & 10K & \cellcolor{TealBlue!30}{1} & \cellcolor{TealBlue!30}{0.00} & \cellcolor{TealBlue!30}{0.0} & \cellcolor{TealBlue!30}{1K} & \cellcolor{TealBlue!30}{1} & \cellcolor{TealBlue!30}{0.00} & \cellcolor{TealBlue!30}{0.0} & \cellcolor{TealBlue!30}{1K} & \cellcolor{TealBlue!30}{1} & \cellcolor{TealBlue!30}{0.00} & \cellcolor{TealBlue!30}{0.0} & \cellcolor{TealBlue!30}{971} & \cellcolor{TealBlue!30}{1} & \cellcolor{TealBlue!30}{0.00} & \cellcolor{TealBlue!30}{0.0} & \cellcolor{TealBlue!30}{937}\\
15 & \texttt{san} & 15 & 10.87 & 12.2 & 1.8M & 11 & 9.80 & 13.3 & 1.9M & 11 & 9.80 & 13.7 & 1.1M & \cellcolor{TealBlue!30}{11} & \cellcolor{TealBlue!30}{9.80} & \cellcolor{TealBlue!30}{10.8} & \cellcolor{TealBlue!30}{0.6M} & 11 & 9.80 & 12.4 & 0.6M\\
7 & \texttt{c-f} & 3 & 10.29 & 0.2 & 9K & 3 & 10.29 & 0.2 & 18K & \cellcolor{TealBlue!30}{3} & \cellcolor{TealBlue!30}{10.29} & \cellcolor{TealBlue!30}{0.1} & \cellcolor{TealBlue!30}{7K} & \cellcolor{TealBlue!30}{3} & \cellcolor{TealBlue!30}{10.29} & \cellcolor{TealBlue!30}{0.1} & \cellcolor{TealBlue!30}{7K} & \cellcolor{TealBlue!30}{3} & \cellcolor{TealBlue!30}{10.29} & \cellcolor{TealBlue!30}{0.1} & \cellcolor{TealBlue!30}{7K}\\
6 & \texttt{ham} & 4 & 9.00 & 26.3 & 2.2M & \cellcolor{TealBlue!30}{3} & \cellcolor{TealBlue!30}{9.00} & \cellcolor{TealBlue!30}{3.1} & \cellcolor{TealBlue!30}{0.3M} & \cellcolor{TealBlue!30}{3} & \cellcolor{TealBlue!30}{9.00} & \cellcolor{TealBlue!30}{3.4} & \cellcolor{TealBlue!30}{0.2M} & 3 & 9.00 & 5.1 & 0.2M & 3 & 9.00 & 5.1 & 0.2M\\
32 & \texttt{gra} & 29 & 40.47 & 24.6 & 2.5M & 28 & 40.47 & 19.8 & 3.5M & \cellcolor{TealBlue!30}{28} & \cellcolor{TealBlue!30}{39.22} & \cellcolor{TealBlue!30}{17.6} & \cellcolor{TealBlue!30}{2.0M} & 28 & 40.47 & 18.9 & 1.5M & 28 & 41.22 & 9.8 & 0.5M\\
4 & \texttt{man} & 3 & 91.00 & 1.1 & 33K & 3 & 91.00 & 1.1 & 51K & \cellcolor{TealBlue!30}{3} & \cellcolor{TealBlue!30}{91.00} & \cellcolor{TealBlue!30}{0.9} & \cellcolor{TealBlue!30}{31K} & \cellcolor{TealBlue!30}{3} & \cellcolor{TealBlue!30}{91.00} & \cellcolor{TealBlue!30}{1.4} & \cellcolor{TealBlue!30}{31K} & \cellcolor{TealBlue!30}{3} & \cellcolor{TealBlue!30}{91.00} & \cellcolor{TealBlue!30}{1.5} & \cellcolor{TealBlue!30}{30K}\\
5 & \texttt{mul} & 5 & 8.40 & 7.2 & 1.8M & 4 & 7.60 & 41.4 & 6.3M & 3 & 7.60 & 24.6 & 2.1M & 3 & 7.60 & 25.1 & 2.1M & \cellcolor{TealBlue!30}{3} & \cellcolor{TealBlue!30}{7.60} & \cellcolor{TealBlue!30}{19.2} & \cellcolor{TealBlue!30}{1.7M}\\
3 & \texttt{fps} & 3 & 105.00 & 0.1 & 7K & 3 & 103.67 & 40.5 & 4.2M & 3 & 103.67 & 56.0 & 3.2M & 3 & 103.67 & 61.0 & 3.2M & \cellcolor{TealBlue!30}{3} & \cellcolor{TealBlue!30}{103.67} & \cellcolor{TealBlue!30}{14.9} & \cellcolor{TealBlue!30}{0.8M}\\
3 & \texttt{zer} & 3 & 44.67 & 11.9 & 2.6M & 3 & 44.67 & 11.4 & 1.6M & 3 & 44.67 & 14.7 & 1.4M & 3 & 44.67 & 13.9 & 1.4M & \cellcolor{TealBlue!30}{3} & \cellcolor{TealBlue!30}{44.67} & \cellcolor{TealBlue!30}{8.2} & \cellcolor{TealBlue!30}{0.9M}\\
3 & \texttt{ini} & \cellcolor{TealBlue!30}{3} & \cellcolor{TealBlue!30}{191.33} & \cellcolor{TealBlue!30}{57.5} & \cellcolor{TealBlue!30}{6.1M} & 3 & 191.33 & 72.7 & 6.1M & 3 & 191.33 & 82.3 & 6.1M & 3 & 191.33 & 82.6 & 6.1M & 3 & 191.33 & 82.6 & 6.1M\\
\hline
5 & \texttt{p2p} & 5 & 38.60 & 1.0 & 8K & 5 & 22.80 & 36.1 & 23K & \cellcolor{TealBlue!30}{2} & \cellcolor{TealBlue!30}{11.80} & \cellcolor{TealBlue!30}{2.8} & \cellcolor{TealBlue!30}{11K} & \cellcolor{TealBlue!30}{2} & \cellcolor{TealBlue!30}{11.80} & \cellcolor{TealBlue!30}{3.1} & \cellcolor{TealBlue!30}{11K} & \cellcolor{TealBlue!30}{2} & \cellcolor{TealBlue!30}{11.80} & \cellcolor{TealBlue!30}{3.4} & \cellcolor{TealBlue!30}{11K}\\
5 & \texttt{ca-} & 5 & 14.40 & 31.6 & 0.2M & 4 & 9.00 & 35.6 & 0.2M & 4 & 1.80 & 99.3 & 0.2M & 3 & 2.60 & 102.3 & 0.2M & \cellcolor{TealBlue!30}{3} & \cellcolor{TealBlue!30}{1.80} & \cellcolor{TealBlue!30}{96.1} & \cellcolor{TealBlue!30}{0.2M}\\
& \multicolumn{21}{c}{balancing constraint: medium}\\3 & \texttt{kel} & 2 & 1.67 & 24.1 & 1.2M & 2 & 0.67 & 35.9 & 1.0M & 2 & 0.67 & 54.7 & 1.0M & 2 & 0.00 & 32.8 & 2K & \cellcolor{TealBlue!30}{2} & \cellcolor{TealBlue!30}{0.00} & \cellcolor{TealBlue!30}{32.1} & \cellcolor{TealBlue!30}{2K}\\
15 & \texttt{p\_h} & 12 & 3.07 & 21.5 & 1.2M & 10 & 1.27 & 24.3 & 0.7M & 11 & 1.27 & 34.4 & 0.6M & \cellcolor{TealBlue!30}{10} & \cellcolor{TealBlue!30}{0.87} & \cellcolor{TealBlue!30}{18.6} & \cellcolor{TealBlue!30}{60K} & \cellcolor{TealBlue!30}{10} & \cellcolor{TealBlue!30}{0.87} & \cellcolor{TealBlue!30}{18.8} & \cellcolor{TealBlue!30}{59K}\\
12 & \texttt{bro} & 9 & 0.83 & 15.6 & 1.9M & \cellcolor{TealBlue!30}{8} & \cellcolor{TealBlue!30}{0.17} & \cellcolor{TealBlue!30}{17.8} & \cellcolor{TealBlue!30}{1.0M} & 8 & 0.17 & 25.4 & 1.0M & 8 & 0.17 & 23.9 & 451 & 8 & 0.17 & 22.2 & 450\\
4 & \texttt{joh} & \cellcolor{TealBlue!30}{1} & \cellcolor{TealBlue!30}{0.00} & \cellcolor{TealBlue!30}{0.0} & \cellcolor{TealBlue!30}{11} & \cellcolor{TealBlue!30}{1} & \cellcolor{TealBlue!30}{0.00} & \cellcolor{TealBlue!30}{0.0} & \cellcolor{TealBlue!30}{11} & \cellcolor{TealBlue!30}{1} & \cellcolor{TealBlue!30}{0.00} & \cellcolor{TealBlue!30}{0.0} & \cellcolor{TealBlue!30}{11} & \cellcolor{TealBlue!30}{1} & \cellcolor{TealBlue!30}{0.00} & \cellcolor{TealBlue!30}{0.0} & \cellcolor{TealBlue!30}{11} & \cellcolor{TealBlue!30}{1} & \cellcolor{TealBlue!30}{0.00} & \cellcolor{TealBlue!30}{0.0} & \cellcolor{TealBlue!30}{11}\\
15 & \texttt{san} & 15 & 8.33 & 35.6 & 5.0M & 7 & 2.67 & 30.4 & 2.4M & 7 & 2.73 & 33.5 & 1.6M & \cellcolor{TealBlue!30}{7} & \cellcolor{TealBlue!30}{1.53} & \cellcolor{TealBlue!30}{42.8} & \cellcolor{TealBlue!30}{0.4M} & \cellcolor{TealBlue!30}{7} & \cellcolor{TealBlue!30}{1.53} & \cellcolor{TealBlue!30}{43.1} & \cellcolor{TealBlue!30}{0.4M}\\
7 & \texttt{c-f} & \cellcolor{TealBlue!30}{3} & \cellcolor{TealBlue!30}{4.14} & \cellcolor{TealBlue!30}{0.0} & \cellcolor{TealBlue!30}{40} & \cellcolor{TealBlue!30}{3} & \cellcolor{TealBlue!30}{4.14} & \cellcolor{TealBlue!30}{0.0} & \cellcolor{TealBlue!30}{40} & \cellcolor{TealBlue!30}{3} & \cellcolor{TealBlue!30}{4.14} & \cellcolor{TealBlue!30}{0.0} & \cellcolor{TealBlue!30}{40} & \cellcolor{TealBlue!30}{3} & \cellcolor{TealBlue!30}{4.14} & \cellcolor{TealBlue!30}{0.0} & \cellcolor{TealBlue!30}{40} & \cellcolor{TealBlue!30}{3} & \cellcolor{TealBlue!30}{4.14} & \cellcolor{TealBlue!30}{0.0} & \cellcolor{TealBlue!30}{40}\\
6 & \texttt{ham} & 4 & 4.67 & 0.2 & 53K & 2 & 4.67 & 0.0 & 1K & \cellcolor{TealBlue!30}{2} & \cellcolor{TealBlue!30}{4.67} & \cellcolor{TealBlue!30}{0.0} & \cellcolor{TealBlue!30}{360} & \cellcolor{TealBlue!30}{2} & \cellcolor{TealBlue!30}{4.67} & \cellcolor{TealBlue!30}{0.0} & \cellcolor{TealBlue!30}{359} & \cellcolor{TealBlue!30}{2} & \cellcolor{TealBlue!30}{4.67} & \cellcolor{TealBlue!30}{0.0} & \cellcolor{TealBlue!30}{359}\\
32 & \texttt{gra} & 26 & 29.28 & 32.8 & 2.7M & 22 & 26.50 & 21.9 & 2.2M & 22 & 24.44 & 23.8 & 1.4M & \cellcolor{TealBlue!30}{22} & \cellcolor{TealBlue!30}{24.25} & \cellcolor{TealBlue!30}{21.5} & \cellcolor{TealBlue!30}{0.9M} & 22 & 24.25 & 23.0 & 0.9M\\
4 & \texttt{man} & 3 & 89.00 & 29.3 & 1.3M & 3 & 88.75 & 44.6 & 1.6M & 3 & 88.75 & 21.1 & 0.6M & \cellcolor{TealBlue!30}{2} & \cellcolor{TealBlue!30}{88.50} & \cellcolor{TealBlue!30}{29.6} & \cellcolor{TealBlue!30}{0.6M} & 2 & 88.50 & 33.4 & 0.6M\\
5 & \texttt{mul} & 5 & 1.20 & 0.3 & 61K & 1 & 1.20 & 0.0 & 1K & 1 & 1.20 & 0.0 & 682 & 1 & 1.20 & 0.0 & 682 & \cellcolor{TealBlue!30}{1} & \cellcolor{TealBlue!30}{1.20} & \cellcolor{TealBlue!30}{0.0} & \cellcolor{TealBlue!30}{560}\\
3 & \texttt{fps} & 3 & 103.00 & 0.0 & 250 & 1 & 102.67 & 0.0 & 429 & 1 & 102.67 & 0.0 & 404 & 1 & 102.67 & 0.0 & 404 & \cellcolor{TealBlue!30}{1} & \cellcolor{TealBlue!30}{102.67} & \cellcolor{TealBlue!30}{0.0} & \cellcolor{TealBlue!30}{261}\\
3 & \texttt{zer} & 3 & 3.33 & 37.4 & 8.2M & \cellcolor{TealBlue!30}{1} & \cellcolor{TealBlue!30}{3.00} & \cellcolor{TealBlue!30}{11.6} & \cellcolor{TealBlue!30}{1.5M} & 1 & 3.00 & 25.4 & 1.5M & 1 & 3.00 & 14.5 & 1.5M & 1 & 3.00 & 14.5 & 1.5M\\
3 & \texttt{ini} & 3 & 189.00 & 0.6 & 65K & 1 & 189.00 & 0.0 & 4K & 1 & 189.00 & 0.0 & 3K & 1 & 189.00 & 0.0 & 3K & \cellcolor{TealBlue!30}{1} & \cellcolor{TealBlue!30}{189.00} & \cellcolor{TealBlue!30}{0.0} & \cellcolor{TealBlue!30}{3K}\\
\hline
5 & \texttt{p2p} & 5 & 35.40 & 1.0 & 8K & 5 & 15.80 & 38.3 & 26K & \cellcolor{TealBlue!30}{1} & \cellcolor{TealBlue!30}{4.40} & \cellcolor{TealBlue!30}{3.3} & \cellcolor{TealBlue!30}{11K} & \cellcolor{TealBlue!30}{1} & \cellcolor{TealBlue!30}{4.40} & \cellcolor{TealBlue!30}{3.5} & \cellcolor{TealBlue!30}{11K} & \cellcolor{TealBlue!30}{1} & \cellcolor{TealBlue!30}{4.40} & \cellcolor{TealBlue!30}{3.8} & \cellcolor{TealBlue!30}{11K}\\
5 & \texttt{ca-} & 5 & 14.40 & 0.7 & 5K & 4 & 8.60 & 2.3 & 8K & 3 & 0.40 & 72.6 & 18K & 2 & 1.20 & 74.1 & 16K & \cellcolor{TealBlue!30}{2} & \cellcolor{TealBlue!30}{0.40} & \cellcolor{TealBlue!30}{64.0} & \cellcolor{TealBlue!30}{15K}\\
& \multicolumn{21}{c}{balancing constraint: loose}\\3 & \texttt{kel} & 2 & 1.67 & 43.3 & 1.8M & 2 & 0.67 & 20.1 & 0.6M & 2 & 0.67 & 30.4 & 0.6M & \cellcolor{TealBlue!30}{2} & \cellcolor{TealBlue!30}{0.00} & \cellcolor{TealBlue!30}{27.7} & \cellcolor{TealBlue!30}{447} & \cellcolor{TealBlue!30}{2} & \cellcolor{TealBlue!30}{0.00} & \cellcolor{TealBlue!30}{28.0} & \cellcolor{TealBlue!30}{419}\\
15 & \texttt{p\_h} & 12 & 2.40 & 20.6 & 1.2M & 10 & 0.73 & 32.1 & 1.0M & 11 & 0.73 & 47.1 & 1.0M & \cellcolor{TealBlue!30}{9} & \cellcolor{TealBlue!30}{0.27} & \cellcolor{TealBlue!30}{18.0} & \cellcolor{TealBlue!30}{3K} & \cellcolor{TealBlue!30}{9} & \cellcolor{TealBlue!30}{0.27} & \cellcolor{TealBlue!30}{18.0} & \cellcolor{TealBlue!30}{3K}\\
12 & \texttt{bro} & 9 & 0.67 & 16.2 & 1.9M & \cellcolor{TealBlue!30}{8} & \cellcolor{TealBlue!30}{0.00} & \cellcolor{TealBlue!30}{10.6} & \cellcolor{TealBlue!30}{0.7M} & 8 & 0.00 & 15.8 & 0.7M & 8 & 0.00 & 13.6 & 264 & 8 & 0.00 & 13.6 & 264\\
4 & \texttt{joh} & \cellcolor{TealBlue!30}{1} & \cellcolor{TealBlue!30}{0.00} & \cellcolor{TealBlue!30}{0.0} & \cellcolor{TealBlue!30}{11} & \cellcolor{TealBlue!30}{1} & \cellcolor{TealBlue!30}{0.00} & \cellcolor{TealBlue!30}{0.0} & \cellcolor{TealBlue!30}{11} & \cellcolor{TealBlue!30}{1} & \cellcolor{TealBlue!30}{0.00} & \cellcolor{TealBlue!30}{0.0} & \cellcolor{TealBlue!30}{11} & \cellcolor{TealBlue!30}{1} & \cellcolor{TealBlue!30}{0.00} & \cellcolor{TealBlue!30}{0.0} & \cellcolor{TealBlue!30}{11} & \cellcolor{TealBlue!30}{1} & \cellcolor{TealBlue!30}{0.00} & \cellcolor{TealBlue!30}{0.0} & \cellcolor{TealBlue!30}{11}\\
15 & \texttt{san} & 15 & 8.20 & 28.1 & 4.0M & 7 & 2.13 & 40.9 & 2.3M & 7 & 2.27 & 29.6 & 1.4M & \cellcolor{TealBlue!30}{5} & \cellcolor{TealBlue!30}{0.27} & \cellcolor{TealBlue!30}{36.8} & \cellcolor{TealBlue!30}{3K} & \cellcolor{TealBlue!30}{5} & \cellcolor{TealBlue!30}{0.27} & \cellcolor{TealBlue!30}{36.8} & \cellcolor{TealBlue!30}{3K}\\
7 & \texttt{c-f} & 2 & 0.71 & 0.0 & 1K & 0 & 0.00 & 0.6 & 98K & \cellcolor{TealBlue!30}{0} & \cellcolor{TealBlue!30}{0.00} & \cellcolor{TealBlue!30}{0.1} & \cellcolor{TealBlue!30}{7K} & \cellcolor{TealBlue!30}{0} & \cellcolor{TealBlue!30}{0.00} & \cellcolor{TealBlue!30}{0.2} & \cellcolor{TealBlue!30}{7K} & \cellcolor{TealBlue!30}{0} & \cellcolor{TealBlue!30}{0.00} & \cellcolor{TealBlue!30}{0.2} & \cellcolor{TealBlue!30}{7K}\\
6 & \texttt{ham} & 4 & 2.00 & 0.0 & 118 & \cellcolor{TealBlue!30}{2} & \cellcolor{TealBlue!30}{2.00} & \cellcolor{TealBlue!30}{0.0} & \cellcolor{TealBlue!30}{118} & \cellcolor{TealBlue!30}{2} & \cellcolor{TealBlue!30}{2.00} & \cellcolor{TealBlue!30}{0.0} & \cellcolor{TealBlue!30}{118} & \cellcolor{TealBlue!30}{2} & \cellcolor{TealBlue!30}{2.00} & \cellcolor{TealBlue!30}{0.0} & \cellcolor{TealBlue!30}{118} & \cellcolor{TealBlue!30}{2} & \cellcolor{TealBlue!30}{2.00} & \cellcolor{TealBlue!30}{0.0} & \cellcolor{TealBlue!30}{118}\\
32 & \texttt{gra} & 23 & 18.97 & 29.4 & 1.9M & 17 & 12.84 & 12.9 & 0.8M & 17 & 12.06 & 15.6 & 0.5M & 17 & 11.56 & 14.5 & 66K & \cellcolor{TealBlue!30}{17} & \cellcolor{TealBlue!30}{11.50} & \cellcolor{TealBlue!30}{17.3} & \cellcolor{TealBlue!30}{0.1M}\\
4 & \texttt{man} & 3 & 88.50 & 38.1 & 0.8M & 3 & 88.50 & 8.5 & 0.3M & 3 & 87.75 & 30.6 & 0.8M & 2 & 87.00 & 43.5 & 0.8M & \cellcolor{TealBlue!30}{2} & \cellcolor{TealBlue!30}{86.50} & \cellcolor{TealBlue!30}{52.4} & \cellcolor{TealBlue!30}{0.8M}\\
5 & \texttt{mul} & 5 & 0.00 & 0.0 & 93 & \cellcolor{TealBlue!30}{0} & \cellcolor{TealBlue!30}{0.00} & \cellcolor{TealBlue!30}{0.0} & \cellcolor{TealBlue!30}{92} & \cellcolor{TealBlue!30}{0} & \cellcolor{TealBlue!30}{0.00} & \cellcolor{TealBlue!30}{0.0} & \cellcolor{TealBlue!30}{91} & \cellcolor{TealBlue!30}{0} & \cellcolor{TealBlue!30}{0.00} & \cellcolor{TealBlue!30}{0.0} & \cellcolor{TealBlue!30}{91} & \cellcolor{TealBlue!30}{0} & \cellcolor{TealBlue!30}{0.00} & \cellcolor{TealBlue!30}{0.0} & \cellcolor{TealBlue!30}{91}\\
3 & \texttt{fps} & 3 & 1.67 & 1.1 & 0.1M & 1 & 1.00 & 12.6 & 1.0M & 1 & 1.00 & 13.4 & 0.5M & 1 & 1.00 & 13.9 & 0.5M & \cellcolor{TealBlue!30}{1} & \cellcolor{TealBlue!30}{0.67} & \cellcolor{TealBlue!30}{53.5} & \cellcolor{TealBlue!30}{2.1M}\\
3 & \texttt{zer} & 3 & 2.00 & 0.2 & 48K & 0 & 1.00 & 2.6 & 0.4M & 0 & 1.00 & 2.7 & 0.2M & 0 & 1.00 & 2.6 & 0.2M & \cellcolor{TealBlue!30}{0} & \cellcolor{TealBlue!30}{1.00} & \cellcolor{TealBlue!30}{0.6} & \cellcolor{TealBlue!30}{58K}\\
3 & \texttt{ini} & 3 & 0.67 & 6.9 & 0.5M & 0 & 0.00 & 20.9 & 1.3M & 0 & 0.00 & 28.9 & 1.0M & 0 & 0.00 & 31.9 & 1.0M & \cellcolor{TealBlue!30}{0} & \cellcolor{TealBlue!30}{0.00} & \cellcolor{TealBlue!30}{11.6} & \cellcolor{TealBlue!30}{0.5M}\\
\hline
5 & \texttt{p2p} & 5 & 33.00 & 1.0 & 8K & 5 & 13.40 & 38.3 & 26K & \cellcolor{TealBlue!30}{1} & \cellcolor{TealBlue!30}{2.00} & \cellcolor{TealBlue!30}{3.4} & \cellcolor{TealBlue!30}{11K} & \cellcolor{TealBlue!30}{1} & \cellcolor{TealBlue!30}{2.00} & \cellcolor{TealBlue!30}{3.5} & \cellcolor{TealBlue!30}{11K} & \cellcolor{TealBlue!30}{1} & \cellcolor{TealBlue!30}{2.00} & \cellcolor{TealBlue!30}{3.9} & \cellcolor{TealBlue!30}{11K}\\
5 & \texttt{ca-} & 5 & 14.40 & 0.7 & 5K & 4 & 8.60 & 2.5 & 8K & 3 & 0.40 & 74.0 & 18K & 2 & 1.20 & 73.9 & 16K & \cellcolor{TealBlue!30}{2} & \cellcolor{TealBlue!30}{0.20} & \cellcolor{TealBlue!30}{69.2} & \cellcolor{TealBlue!30}{16K}\\
\end{tabular}
\end{scriptsize}

\end{table}
\end{center}
Then, we compute (uniformly at random) a balanced 4-partition $\{\aset_1,\aset_2,\aset_3,\aset_4\}$ of 
the vertices and we post the following constraints:
$
\max(\{|\aset_i \cap \asetvar| \mid 1 \leq i \leq 4 \}) - \min(\{|\aset_i \cap \asetvar| \mid 1 \leq i \leq 4 \}) \leq \balance
$.
%
%
%
For each graph instance, we generated 3 instances for $\balance \in \{0,4,8\}$ denoted ``tight'', ``medium'' and ``loose'' respectively. However, the classes \texttt{p2p} and \texttt{ca-} are much too large for these values to make sense. In this case we used three ratios $0.007, 0.008$ and $0.009$ of the number of nodes instead.
%

We compared 5 methods, all implemented in Mistral~\cite{mistral} and ran on CORE I7 processors with a time limit of 5 minutes:



	{\textbf{\decomp}} is a simple decomposition in 2-clauses and a cardinality constraint.
	%
	{\textbf{\onlylb}} uses only Buss kernelization and the clique cover lower bound.
		It corresponds to non-colored lines in Algorithm~\ref{algo::vcpropag}. The witness is initialised to $\vertices$ and never changes, 
		and Line~\ref{ln:rkernel} is replaced by a simple identity $\truekernelized{\residue} \gets \rigikernelized{\residue}$.
	 %
	{\textbf{\kernelp}} uses kernelization, but no witness cover. It corresponds to 
	Algorithm~\ref{algo::vcpropag} minus the instruction line~$11$, 
	with $\limit$ set to 0.
	%
	{\textbf{\witness}} uses kernelization, and the witness cover for the lower bound $\lb{\costvar}$.
	It corresponds to
	Algorithm~\ref{algo::vcpropag} minus the instruction line~$11$,	
	 with $\limit$ set to 5000.
	%
	{\textbf{\all}} is Algorithm~\ref{algo::vcpropag} with $\limit$ set to 5000.





The results of these experiments are reported in Table~\ref{tab:bvc}. 
Instances are clustered by classes whose cardinality is given in the first column. 
These classes are ordered from top to bottom by decreasing ratio of minimum vertex cover size over number of nodes. 
We report four values for each class and each method: `\#s' is the number of instances of the class that were not solved to optimality, `gap' is the average gap w.r.t. the smallest vertex cover found, 
`cpu' and `\#nd' are mean CPU time in seconds and number of nodes visited, respectively, until finding the best solution. Notice that CPU times and number of nodes are then only comparable when the objective values (gaps) are equal.
We color the tuples $\langle$\#s, gap, cpu, \#nd$\rangle$ that are lexicographically minimum for each class\footnote{With a ``tolerance'' of 1s and 1\% nodes.}.

Instances with same value of $\balance$ are grouped in the same sub-table.
The ``shift'' of colored cells from left to right when going from top to bottom in each subtable was to be expected since the kernelization is more effective on instances with small vertex cover.
It should be noted that many instances from the \dimacs repository are extremely adverse to our method as they tend to have very large vertex covers. On the other hand, kernelization is very effective on large graphs from \snap.  

We can also observe another shift of colored cells from left to right when moving to a subtable to the next. This was also an expected outcome since the pruning on this constraint becomes more prevalent when the problem is closer to pure vertex cover.

Last, we can observe that every reasoning step ($0$-loss-less kernels, lower bound from the witness and pruning from the witness) improves the overall results.

\section{Conclusion}
\label{sec:conclu}

We have shown that the kernelization techniques can be an effective way to reason about NP-hard constraints that are fixed parameter tractable.
In order to design a propagation algorithm we introduced the notion of loss-less kernel and outlined several ways to benefit from a small kernel.
Our experimental evaluation on the \vc constraint shows the promise of this approach.

\clearpage

\bibliographystyle{plain}
\bibliography{bibfile}

\begin{thebibliography}{10}

\bibitem{abu2007crown}
Faisal~N Abu-Khzam, Michael~R Fellows, Michael~A Langston, and W~Henry Suters.
\newblock Crown structures for vertex cover kernelization.
\newblock {\em Theory of Computing Systems}, 41(3):411--430, 2007.

\bibitem{networkflow:book}
R.~K. Ahuja, T.~L. Magnanti, and J.~B. Orlin.
\newblock {\em {Network Flows: Theory, Algorithms, and Applications}}.
\newblock Prentice-Hall, Inc., 1993.

\bibitem{DBLP:conf/cp/Beldiceanu01}
Nicolas Beldiceanu.
\newblock Pruning for the minimum constraint family and for the number of
  distinct values constraint family.
\newblock In {\em CP}, pages 211--224, 2001.

\bibitem{bessiere08:par}
C.~Bessiere, E.~Hebrard, B.~Hnich, Z.~Kiziltan, C.-G. Quimper, and T.~Walsh.
\newblock The {P}arameterized {C}omplexity of {G}lobal {C}onstraints.
\newblock In {\em AAAI}, pages 235--240, 2008.

\bibitem{bessiere06:fil}
C.~Bessiere, E.~Hebrard, B.~Hnich, Z.~Kiziltan, and T.~Walsh.
\newblock Filtering {A}lgorithms for the {NV}alue {C}onstraint.
\newblock {\em Constraints}, 11(4):271--293, 2006.

\bibitem{buss1993nondeterminism}
Jonathan~F Buss and Judy Goldsmith.
\newblock Nondeterminism within p\^{}*.
\newblock {\em SIAM Journal on Computing}, 22(3):560--572, 1993.

\bibitem{conf/mfcs/ChenKX06}
J.~Chen, I.~A. Kanj, and G.~Xia.
\newblock {Improved Parameterized Upper Bounds for Vertex Cover}.
\newblock In {\em MFCS}, pages 238--249, 2006.

\bibitem{Chlebík2008292}
Miroslav Chleb\'ik and Janka Chleb\'ikov\'a.
\newblock Crown reductions for the minimum weighted vertex cover problem.
\newblock {\em Discrete Applied Mathematics}, 156(3):292--312, 2008.

\bibitem{Damaschke2006337}
Peter Damaschke.
\newblock {Parameterized enumeration, transversals, and imperfect phylogeny
  reconstruction}.
\newblock {\em {Theoretical Computer Science}}, 351(3):337--350, 2006.

\bibitem{fages13:fil}
J.-G. Fages and T.~Lap{\`e}gue.
\newblock {Filtering AtMostNValue with Difference Constraints: Application to
  the Shift Minimisation Personnel Task Scheduling Problem}.
\newblock In {\em CP}, pages 63--79, 2013.

\bibitem{Garey:1979:CIG:578533}
Michael~R. Garey and David~S. Johnson.
\newblock {\em {Computers and Intractability: A Guide to the Theory of
  NP-Completeness}}.
\newblock W. H. Freeman \& Co., New York, NY, USA, 1979.

\bibitem{DBLP:conf/ijcai/GaspersS11}
S.~Gaspers and S.~Szeider.
\newblock {Kernels for Global Constraints}.
\newblock In {\em IJCAI}, pages 540--545, 2011.

\bibitem{mistral}
E.~Hebrard.
\newblock {M}istral, a {C}onstraint {S}atisfaction {L}ibrary.
\newblock In {\em The {T}hird {I}nternational {CSP} {S}olver {C}ompetition},
  pages 31--40, 2008.

\bibitem{hopcroft73:alg}
J.~E. Hopcroft and R.~M. Karp.
\newblock {An n$^{\mbox{5/2}}$ Algorithm for Maximum Matchings in Bipartite
  Graphs}.
\newblock {\em SIAM J. Comput.}, 2(4):225--231, 1973.

\bibitem{nemhauser1975vertex}
George~L Nemhauser and Leslie~E Trotter~Jr.
\newblock Vertex packings: structural properties and algorithms.
\newblock {\em Mathematical Programming}, 8(1):232--248, 1975.

\bibitem{regin94:fil}
J.-C. R{\'e}gin.
\newblock A {F}iltering {A}lgorithm for {C}onstraints of {D}ifference in
  {CSP}s.
\newblock In {\em AAAI}, pages 362--367, 1994.

\bibitem{samer2008backdoor}
Marko Samer and Stefan Szeider.
\newblock Backdoor trees.
\newblock In {\em AAAI}, volume~8, pages 13--17, 2008.

\bibitem{hoeve06:glo}
W.-J. van Hoeve, G.~Pesant, and L.-M. Rousseau.
\newblock {On Global Warming: Flow-Based Soft Global Constraints}.
\newblock {\em Journal of Heuristics}, 12(4-5):347--373, 2006.

\end{thebibliography}

\end{document}